\documentclass{article}
\usepackage{amsfonts}
\usepackage{amsthm}
\usepackage{url}
\usepackage{graphicx}
\newtheorem{proposition}{Proposition}
\newcommand{\bu}{\mathbf{u}}

\newcommand{\bv}{\widehat{\bf u}}
\newcommand{\bvr}{\widehat{\bf \mathsf u}}
\newcommand{\bw}{\mathbf{v}}

\usepackage{microtype}
\usepackage{graphicx}
\usepackage{subfigure}
\usepackage{booktabs} 

\usepackage{hyperref}

\usepackage[accepted]{icml2019}

\icmltitlerunning{Automatic Classifiers as Scientific Instruments}

\begin{document}

\twocolumn[
\icmltitle{Automatic Classifiers as Scientific Instruments:\\One Step Further Away from Ground-Truth}



\icmlsetsymbol{equal}{*}

\begin{icmlauthorlist}
\icmlauthor{Jacob Whitehill}{wpi}
\icmlauthor{Anand Ramakrishnan}{wpi}
\end{icmlauthorlist}

\icmlaffiliation{wpi}{Department of Computer Science, Worcester Polytechnic Institute (WPI), MA, USA}

\icmlcorrespondingauthor{Jacob Whitehill}{jrwhitehill@wpi.edu}

\icmlkeywords{Pearson correlation}

\vskip 0.3in
]



\printAffiliationsAndNotice{}  

\begin{abstract}
Automatic machine learning-based detectors of various psychological and social 
phenomena (e.g., emotion, stress, engagement)
have great potential to advance basic science.
However, when a detector $d$ is trained to approximate an existing measurement tool
(e.g., a questionnaire, observation protocol), then care must
be taken when interpreting measurements collected using $d$ since they are one step further removed from the underlying construct.
We examine how the accuracy of $d$, as quantified by the correlation $q$ of $d$'s outputs with the ground-truth 
construct $U$, impacts the estimated correlation between $U$ (e.g., stress) and some other phenomenon $V$ (e.g., academic performance).
In particular: (1) We show that if the true correlation between $U$ and $V$ is $r$, then the expected
sample correlation, over all vectors $\mathcal{T}^n$ whose correlation with $U$ is $q$, is $qr$.
(2) We derive a formula for
the probability that the sample correlation (over $n$ subjects) using  $d$ is positive given that the true correlation  is negative
(and vice-versa);  this probability can be substantial (around $20-30\%$)
for values of $n$ and $q$  that have been used in recent affective computing studies.
(3) With the goal to reduce the variance
of correlations estimated by an automatic detector, we show that training multiple neural networks
$d^{(1)},\ldots,d^{(m)}$ using different training architectures and hyperparameters   for the same detection task
provides only limited ``coverage'' of $\mathcal{T}^n$.
\end{abstract}

\section{Introduction}
Automatic classifiers have the potential to advance basic research
in psychology, education, medicine, and many other fields by serving as \emph{scientific instruments}
that can measure behavioral, medical, social, and other phenomena with higher temporal resolution,
lower cost, and greater consistency than is possible with traditional methods such
as human-coded questionnaires or observation protocols.
The affective computing \cite{picard2010affective} community is starting to see some first fruits of this potential:
Perugia, et al.~\cite{perugia2017electrodermal} used the Empatica E4 wristband sensor
to explore the relationship  between participants' ($n=14$) electrodermal activity (EDA) and
their emotional states when playing cognitive games.
Parra, et al.~\cite{parra2017multimodal} used the Emotient facial expression recognition
software to identify a positive correlation ($r=0.32$, $n=59$ participants) 
between emotions and adult attachment \cite{collins1990adult}.
Chen, et al.~\cite{chen2014initial} used Emotient
in a study of how facial emotion is associated with job interview performance among $n=4$ participants.

In most empirical studies designed to measure the relationship between two phenomena
$U$ and $V$ (e.g., engagement \cite{monkaresi2017automated},
grit \cite{duckworth2007grit}, stress, attachment \cite{collins1990adult},
academic performance, etc.), the investigator chooses a validated instrument for each
phenomenon and records measurements of each variable for $n$ participants. She/he then
computes a statistic, such as the Pearson product-moment coefficient, that captures
the magnitude and sign (as well as statistical significance) of the relationship between
the two variables. Machine learning
offers the potential to create a new array of scientific instruments with important advantages
compared to standard measurement tools. However, they also bring a potential pitfall that --
while not fundamentally new, i.e., there is always a separation between a construct and its measurement --
is exacerbated compared to using standard measurements: If one creates a new scientific
instrument by training an automatic detector $d$ to mimic a standard instrument as closely as possible, 
then  $d$ is \emph{one degree of separation further removed from the underlying phenomenon $U$} --
i.e., it is an estimator of another estimator.

{\bf Motivating example}:
Suppose a behavioral scientist wishes to examine the relationship between stress (construct
$U$) and academic performance (construct $V$).
Using a traditional approach, she/he could conduct an experiment in which each participant completes some
cognitively demanding task and then takes a test. To measure stress, the scientist
could also ask each participant to complete an established survey, e.g., the Dundee State Stress Questionnaire
\cite{matthews1999validation}. The relationship between $U$ and $V$ could then be estimated
as the correlation $r=\rho(\bu,\bw)$ between the vector of test scores $\bw$ and the corresponding vector
of stress measurements $\bu$ over all $n$ participants.

However, suppose that the researcher also has access to an \emph{automatic stress detector} $d$ that
uses the participant's face pixels
to measure his/her stress level. Suppose that the accuracy of $d$ was previously  validated w.r.t.
a standard stress questionnaire (like \cite{matthews1999validation}), and the validation showed that
the outputs of $d$, which we denote with $\bv$,
have an expected correlation of $q$ with the standard questionnaire.
What could go wrong, in terms of spurious deductions, when the correlation between $U$ and $V$
is estimated as $\rho(\bv,\bw)$ instead of $\rho(\bu,\bw)$?

Figure \ref{fig:example} shows one hypothetical example of what can go wrong:
vectors $\bu,\bw \in\mathbb{R}^n$ contain  measurements from $n$ participants
of constructs $U$ and $V$, respectively, where $\bu$ is obtained through a standard instrument.
The Pearson product-moment correlation between two vectors can be written:
\[
\rho(\bu,\bw) = \frac{\left(\bu - \mu_{\bu} \right)^\top \left(\bw - \mu_{\bw}\right)}{\|\bu - \mu_{\bu}\|_2 \|\bw - \mu_{\bw}\|_2}
\]
where $\mu_{\bu}$ (or $\mu_{\bw}$) is a vector  whose elements equal the mean value of 
$\bu$ (or $\bw$). Defined in this way, the value of $\rho$ is random if either of its two
arguments is random.
If $\bu$ and $\bw$ are both normalized to have 0-mean and unit-length, then their correlation depends
only on the \emph{angle} between them:
\[
\rho(\bu,\bw) = \bu^\top \bw = \cos \angle(\bu,\bw)
\]
In  the figure, this correlation
is $\cos(105^\circ) \approx -.259$, i.e., the data
suggest that $U$ is \emph{negatively} correlated with $V$.
Suppose instead that  the researcher had used an automatic detector $d$ to obtain $\bv$,
where prior analysis had  established that the expected correlation of $d$'s outputs and
the standard instrument was $q= \cos(30^\circ)\approx 0.866$. If the researcher uses the correlation
$\rho(\bv,\bw)$ to estimate the relationship between $U$ and $V$, then she/he would
obtain $\cos(135^\circ) = -0.707$ -- a much larger magnitude, but at least the same sign as,
the $-0.259$ correlation obtained using a standard instrument for $U$. But the bigger problem is the following:
$\bv$ is not the only
vector whose correlation with the ``ground-truth'' measurements $\bu$  is $q$. Vector $\bv'$
also has the same correlation. If the researcher obtained measurements $\bv'$, then she/he
would deduce a \emph{positive} correlation of $\rho(\bv',\bw)=\cos(75^\circ) \approx 0.259$ -- 
this is opposite to the correlation obtained with a standard instrument.




\begin{figure}
\begin{center}
\includegraphics[trim={0 0 0 3cm},clip,width=1.475in]{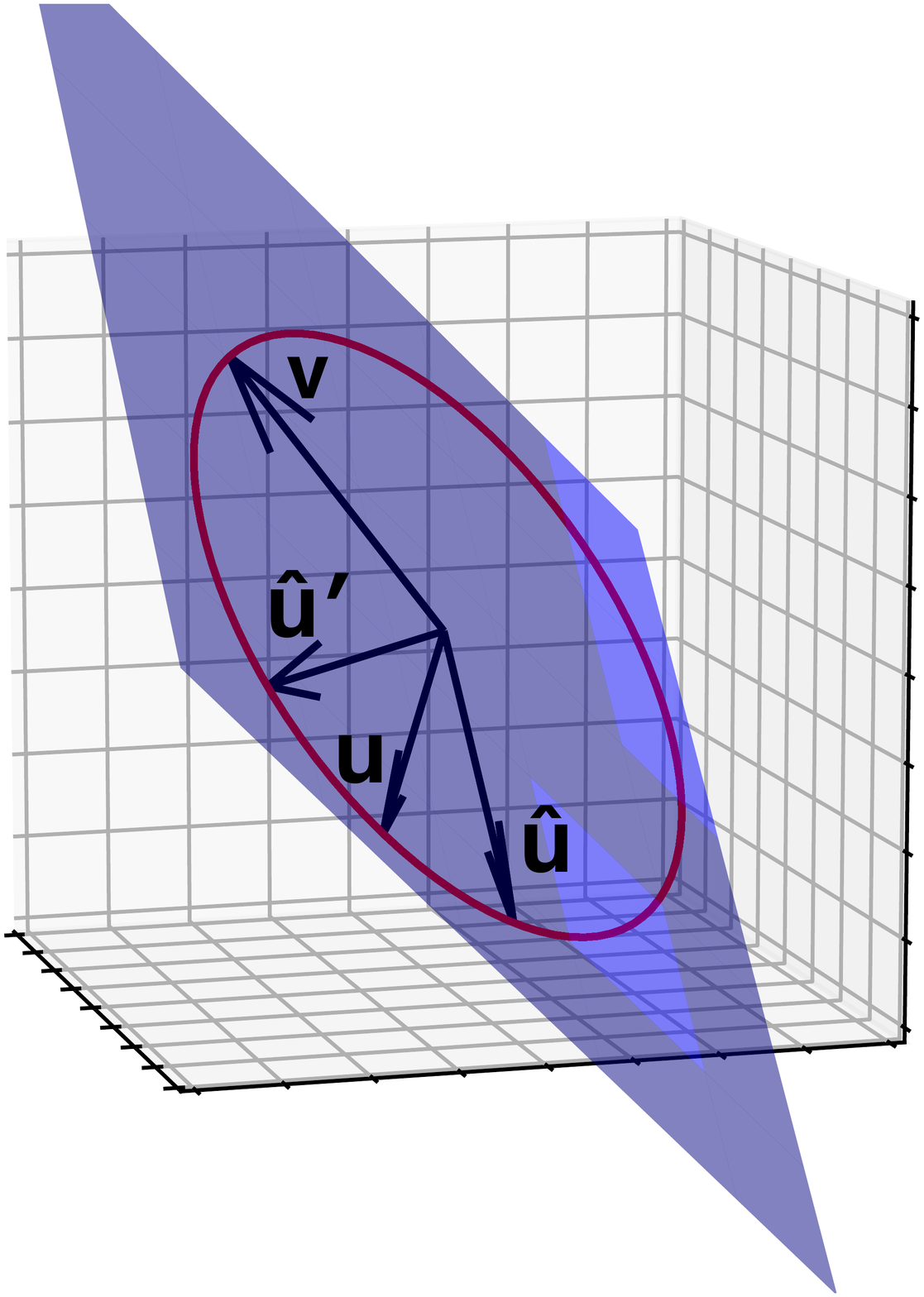}
\includegraphics[trim={1cm 0 0 0},clip,width=1.725in]{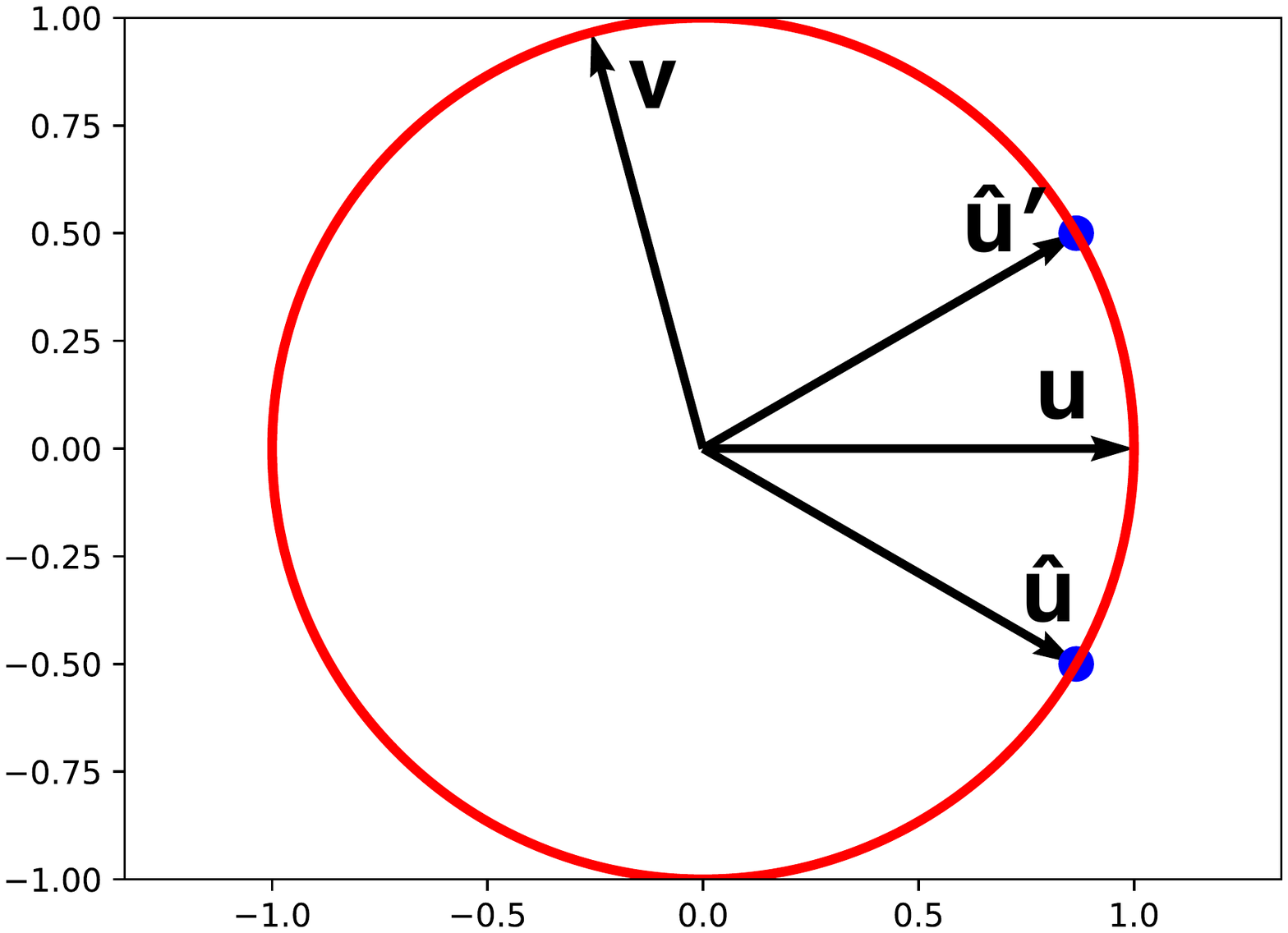}
\end{center}
\caption{$\bu$ and $\bw$ are measurements of some behavioral, social, educational, (or other) phenomenoma
for $n$ participants in an experiment.
$\bv$ (or $\bv'$) are proxy measurements of $\bu$ that were obtained from an automatic detector.
Both $\bv$ and $\bv'$ have the same correlation ($r\approx 0.867$) with $\bu$. However,
depending on \emph{which} vector is obtained from the detector, the estimated correlation 
can be very different: $\rho(\bv,\bw)<0$, but $\rho(\bv',\bw)>0$.
}
\label{fig:example}
\end{figure}

In this paper  we explore how the accuracy $q$ of a scientific instrument $d$, as measured by the Pearson correlation with
the ground-truth construct $U$, impacts the estimated correlation between constructs $U$ and $V$.
Although there are various ways of quantifying the relationship between two vectors of measurements (e.g., RMSE, MAE),
the Pearson correlation is one of the most commonly used metrics.
{\bf Contributions}: (1) We prove that  $\textrm{E}[\rho(\bvr, \bw)]=qr$, where $r$ is the true correlation between $U$ and $V$ and random vector $\bvr$ is 
sampled uniformly over the $(n-3)$-sphere $\mathcal{T}^n$ of
$0$-mean unit-vectors whose correlation with $\bu$ is $q$.
Next, as one of the most fundamental aspects of the relationship between two variables is whether they are positively
or negatively correlated, 
(2) we derive a function $h$ to compute the probability that the sample correlation
(over $n$ subjects) using $d$ is positive,
given that the true correlation between $U$ and $V$ is negative (and vice-versa). 
We also prove that $h$ is monotonically decreasing in $n$ and in $q$, i.e., the danger of a false
correlation is mitigated by training a more accurate detector or collecting data from more participants.
Finally, 
(4) we explore to what extent the sphere $\mathcal{T}^n$ can be ``covered'' by 
measurement vectors $\bvr^{(1)},\ldots,\bvr^{(m)}$ obtained by training $m$ different
neural networks  on the same dataset for the same detection task but using different
configurations (e.g., architectures, hyperparameters, etc.). We also devise a novel technique to visualize this coverage of $\mathcal{T}^n$.

\section{Related Work}
The issue of how product-moment (Pearson)
correlations among a subset of variables constrain the possible correlations among the remaining
variables has interested statisticians since the 1960s. While there has been significant prior work
on the trivariate case in particular, we
are not aware of any work that proves exactly the same results as what we present here.
\cite{priest1968range} showed a lower bound on the mean intercorrelation between variables.
\cite{glass1970geometric}, and also  \cite{leung1975note},
proved that, in trivariate distributions, there are range restrictions on the possible correlations
between $U$ and $V$ when the correlations between $V$ and $W$ and between $U$ and $W$ are already known.
\cite{olkin1981range} extended this result to multivariate distributions beyond 3 variables.

More recent, and most similar to our work, is a study by  \cite{carlson2012understanding}
from the operations research community in 2012. They examined the methodological risk
of using proxy measures to estimate the correlations between different constructs. Using analytical
results by  \cite{leung1975note}, they show how the observed correlations
between $\bu$ and $\bv$ can vary substantially as a function of the reliability
of a proxy measure $\widehat{\bu}$ of $\bu$. In contrast to our work, theirs is based on simulations and contains
no formal proofs.

Finally,  which vector of measurements $\bvr \in \mathcal{T}^n$ is obtained could potentially be related to subgroup membership -- e.g., gender, ethnicity, age. From the perspective of \emph{fairness} in machine learning \cite{barocas2017fairness,kearns2018empirical}, it could therefore be  important to understand how much variation there is among different subgroup populations over the different $\bvr$ that are obtained from the scientific instrument.

\section{Modelling assumptions}
{\bf Notation}: We typeset random variables in \textsf{Futura} font;
all other variables are fixed (non-random).
$\mathbf{u}$ and $\mathbf{v}$ are fixed $n$-vectors of ground-truth measurements of
two phenomena $U$ and $V$, respectively. 
$\hat{\mathsf u}$ is a random $n$-vector (sampled from $\mathcal{T}^n$),
obtained from scientific instrument $d$, containing noisy measurements of $\mathbf{u}$.

For our theoretical results (Propositions  \ref{prop:qr} through \ref{prop:monotone2}), we assume the correlation between
$\hat{\mathsf u}$ and $\mathbf{u}$ is exactly $q$; then, $\hat{\mathsf u}$ is sampled uniformly from the sphere $\mathcal{T}^n$ of all such vectors.
However, for our empirical results regarding the probability of a ``false correlation'' (see Section \ref{sec:false_corr}),
we relax this assumption: It is unlikely that instrument $d$ (that we assume was previously estimated to have correlation $q$ w.r.t.~ground-truth)
always produces a vector $\hat{\mathsf u}$ whose correlation with $\mathbf{u}$ is \emph{exactly} $q$. Instead, the actual correlation $\hat{\mathsf q}$ of
$\hat{\mathsf u}$ and $\mathbf{u}$ comes from a sampling distribution of Pearson correlations (with ``true'' correlation $q$ and the number of subjects
$n$ as parameters; see Section \ref{sec:sampling_dist}). 
Then, given a fixed $\hat{\mathsf q}$, we sample $\hat{\mathsf u}$ from the corresponding $\mathcal{T}^n$
(which depends on $\hat{\mathsf q}$). When we compute the probability of a ``false correlation'' in our case studies 
(Section \ref{sec:case_studies}), we marginalize over $\hat{\mathsf q}$.


\section{Expected Correlation of $\bvr \in \mathcal{T}^n$ with $\bw$}
\label{sec:qr}
When we use an automatic classifier $d$ to obtain
a vector of measurements, then we obtain a vector $\bv$ whose correlation
with the underlying construct $U$ (e.g., stress) is $q$. However, as illustrated
in the example above, there can be multiple such vectors, and which one is obtained
can make a big difference on the estimated correlation.
As we show below, the set of $0$-mean unit-length vectors with a fixed correlation 
to another unit-vector is an $(n-3)$-sphere embedded in $\mathbb{R}^n$. If we sampled uniformly at random
from this sphere, then what would be 
the expected sample correlation 
between $\bvr$ and some other vector $\bw$ (e.g., academic performance)?

To simplify our analyses below, we assume $\bu,\bv,\bw$ all have $0$-mean and unit-length 
since Pearson correlation is invariant to these quantities.
(Regarding the uniformity assumption: see Future Work in Section \ref{sec:conclusions}.)


\begin{proposition}
\label{prop:qr}
Let $\bu,\bw$ be $n$-dimensional, $0$-mean, unit-length vectors with 
a Pearson product-moment correlation $\rho(\bu,\bw)=r$. Then (1) the 
set $\mathcal{T}^n$ of $0$-mean, unit-length vectors whose
correlation with $\bu$ is $q$ is an $(n-3)$-sphere embedded in $\mathbb{R}^n$.
Moreover, (2) if $\bvr$ 
is a random vector  sampled uniformly from $\mathcal{T}^n$, then
the expected sample correlation $\textrm{E}[\rho(\bvr,\bw)]=qr$.
\end{proposition}
\begin{proof}
The set of all $0$-mean $n$-vectors constitutes a hyperplane
\[
\mathcal{H} \doteq \{ \mathbf{x} \in \mathbb{R}^n: \mathbf{1}^\top \mathbf{x}=0 \}
\]
that passes through the origin with normal vector $\mathbf{1} \doteq (1,\ldots, 1)$.
The set of all unit-length vectors constitutes an $(n-1)$-sphere
\[
\mathcal{S}^{n-1} \doteq \{ \mathbf{x} \in \mathbb{R}^n: \|\mathbf{x}\|_2 = 1 \}
\]
embedded in $\mathbb{R}^n$.
Therefore, $\bu,\bw,\bv \in \mathcal{H} \cap \mathcal{S}^{n-1}$. 
Figure \ref{fig:example} (left) shows $\mathcal{H}$ in blue, as well as the
intersection of $\mathcal{H}$ with $\mathcal{S}^{n-1}$ as a red circle.
Since all three vectors have $0$-mean and unit-length,
then the correlations between these vectors depend only on the angles between them. Hence,
w.l.o.g.~we can rotate the axes 
so that $\mathcal{H}$ consists of all vectors whose first coordinate is $0$,
and all correlations will be preserved.
After doing so, 
the only remaining constraint is that the projected vectors  have unit-length.

More precisely, we can compute an orthonormal basis $\mathbf{B}$ such that
the first coordinate of vector $\mathbf{B} \mathbf{x}$ is 0 for every $\mathbf{x} \in \mathcal{H}$, and
such that
\begin{eqnarray}
\mathbf{B} \bu &=& (0,1,\overbrace{0,0,\ldots,0}^{n-2}) \label{eqn:transformed_u}\\
\mathbf{B} \bw &=& (0,a,b,0,\ldots,0)
\end{eqnarray}
Geometrically, this means that we can define $\mathbf{B}$ so that the projected 
$\bu,\bw$ lie in the plane spanned by the second and third vectors in basis $\mathbf{B}$ (see
Figure \ref{fig:example} (right)) -- this makes
the rest of the derivation much simpler. $a$ represents the component of $\bw$ parallel to $\bu$,
and $b$ is the component orthogonal to $\bu$.
Since the correlation between $\bu$ and $\bw$ is $r$, and since $\mathbf{B}$ is orthonormal, then
\begin{eqnarray*}
\lefteqn{\left(\mathbf{B} \bu\right)^\top \left(\mathbf{B} \bw\right) = \bu^\top \bw} &&\\
         &=& r \\
         &=& 0\times0 + 1\times a + 0\times b + 0 + \ldots + 0 \\
	 &=& a
\end{eqnarray*}
and hence $a=r$.
Since $\|\bw\|_2=\|\mathbf{B} \bw \|_2 = 1$, then  $b=\sqrt{1-r^2}$.

Now consider any vector $\bv$ whose correlation with $\bu$ is $q$.
Let us define $(\hat{u}_1, \ldots, \hat{u}_n) \doteq \mathbf{B} \bv$. By construction
of $\mathbf{B}$, we already know that $\hat{u}_1=0$. We also have
\begin{eqnarray*}
\lefteqn{\left(\mathbf{B} \bv\right)^\top \left(\mathbf{B} \bu\right) = \bv^\top \bu} &&\\
         &=& q \\
         &=& \hat{u}_1 \times 0 + \hat{u}_2 \times 1 + \hat{u}_3 \times 0 + \ldots + \hat{u}_n \times 0 \\
	 &=& \hat{u}_2
\end{eqnarray*}
and hence $\hat{u}_2 = q$. Since $\mathbf{B} \bv$ is a unit-vector, then
\[
\mathbf{B} \bv \in \left\{ (0, q, \hat{u}_3, \ldots, \hat{u}_n): \sum_{i=3}^n \hat{u}_i^2 = 1 - q^2 \right\}
\]
This set is the surface of an $(n-3)$-sphere, with radius
$\sqrt{1-q^2}$, embedded in $\mathbb{R}^n$. Since $\mathbf{B}$ simply rotates the axes, then $\mathcal{T}^n$ 
is likewise a $(n-3)$-sphere embedded in $\mathbb{R}^n$. This proves part 1.

For part 2: 
When sampling uniformly from $\mathcal{T}^n$,
the distribution of $\hat{\mathsf u}_3$ on the $(n-3)$-sphere is symmetrical about 0.
Then $\textrm{E}[\hat{\mathsf u}_3]=0$, and hence:
\begin{eqnarray*}
\textrm{E}[\rho(\bvr, \bw)] &=& \textrm{E}[\bvr^\top \bw] \\
             &=& \textrm{E}[(\mathbf{B} \bvr)^\top (\mathbf{B} \bw)] \\
             &=& \textrm{E}[0 + q\times r + \hat{\mathsf u}_3\times \sqrt{1-r^2}] \\
             &=& qr + \textrm{E}[\hat{\mathsf u}_3] \sqrt{1-r^2} \\
	     &=& qr
\end{eqnarray*}
\end{proof}

\noindent {\bf Example}: For the case $n=3$, consider the four vectors shown in Figure \ref{fig:example} (left)
whose values are approximately:
\[
\begin{array}{ll}\bu = (.816, -.408, -.408) & \bw = (-.211, -.577, .788)\\
                 \bv = (.707, 0, -.707) & \bv' = (.707, -.707, 0)  \end{array}
\]
By construction, $\rho(\bv,\bu)=\rho(\bv',\bu)=\cos(30^\circ)=q$,
and $\rho(\bu,\bw)=\cos(105^\circ)=r$.
Via a change of basis $\mathbf{B}$, the vectors can be rotated so that
\begin{eqnarray*}
\mathbf{B} \bu &=& (0, 1, 0)\\
\mathbf{B} \bw &=& (0, \cos(105^\circ), \sin(105^\circ))\\
\mathbf{B} \bv &=& (0, \cos(30^\circ), -\sin(30^\circ))\\
\mathbf{B} \bv' &=& (0, \cos(30^\circ), \sin(30^\circ))
\end{eqnarray*}
The rotated vectors are shown in Figure \ref{fig:example} (right).  The set $\mathcal{T}^3$
contains exactly two elements (since it is a $0$-sphere): $\bv$ and $\bv'$. If $\bvr$ is sampled uniformly at 
random from $\mathcal{T}^3$, then $\textrm{E}[\rho(\bvr, \bw)]=qr\approx -.224$. 
This result agrees with
\begin{eqnarray*}
\frac{1}{2} \left[ \rho(\bv,\bw) + \rho(\bv',\bw)\right] &=& \frac{1}{2} \left[ \cos(135^\circ)+ \cos(75^\circ)\right] \\ & \approx & -.224
\end{eqnarray*}

\section{Probability of false correlations}
\label{sec:false_corr}
One of the most fundamental distinctions is whether two phenomena are positively or
negatively correlated with each other (or neither). What is the probability that
$\rho(\bvr, \bw)\geq 0$ given that the correlation $r<0$ (\emph{false positive correlation}); or that
$\rho(\bvr, \bw)< 0$ given that the correlation $r\geq 0$ (\emph{false negative correlation})?
How do these probabilities change as $n$ increases or $q$ increases?
The proofs of the following propositions are given in the appendix.

\begin{proposition}
\label{prop:deviation}
Let $q \in (0,1]$ be the correlation between the detector's output $\bvr$ and ground-truth $\bu$; let $r$ be the 
correlation between $\bu$ and $\bw$; and let $\bvr$ be sampled uniformly from $\mathcal{T}^n$.  If $r <0$, then
the probability of a \emph{false positive} correlation (in the sense defined above) 
is given by the function 
\[ h(n,q,r) =
\left\{ \begin{array}{ll}\frac{1}{2} \textrm{I}[c^2 \leq 1-q^2] & n=3 \\
 \frac{1}{2} \int_0^\infty f_1(t) F_{n-3}\left(\frac{1 - q^2 - c^2}{c^2} t\right) dt & n>3\end{array}\right.\]
where $\textrm{I}[\cdot]$ is the 0/1 indicator function, $f_k$ and $F_k$ are the PDF and CDF of
a $\chi^2$-random variable with $k$ degrees of freedom, and
$c = |qr|/\sqrt{1-r^2}$.  If $r > 0$, then
the probability of a \emph{false negative} correlation 
is also given by $h$.
\end{proposition}


\begin{proposition}
\label{prop:monotone}
For every fixed $c>0$ and $q \in (0,1]$, function $h$ is monotonically decreasing in $n$.
\end{proposition}

\begin{proposition}
\label{prop:monotone2}
For every fixed $n> 3$, function $h$ is monotonically decreasing in $q \in (0,1]$.
\end{proposition}

Propositions \ref{prop:monotone} and \ref{prop:monotone2} imply that the probability of a false
correlation diminishes as $n$ increases or $q$ increases.

\subsection{Marginalizing over the sampling distribution of $\hat{\mathsf q}$}
\label{sec:sampling_dist}
Up to now we have glossed over the important detail that, when a scientific instrument
with \emph{average} accuracy (Pearson correlation) of  $q$ is used to obtain measurements for  $n$ subjects, the
correlation of the sampled $\bv$ with their ground-truth values need not be exactly $q$;
rather, the actual correlation $\hat{\mathsf q}$ is drawn from a sampling distribution $\mathrm{Pr}(\hat{\mathsf q}\ |\ q, n)$.
Particularly for small $n$, $\hat{\mathsf q}$ can deviate substantially from $q$.
Hence, to compute the probability of a false correlation, it is necessary to marginalize (via numeric integration)
over $\hat{\mathsf q}$:
\[ \int_{\hat{\mathsf q}} h(n,\hat{\mathsf q},r) \mathrm{Pr}(\hat{\mathsf q}\ |\ q, n) d\hat{\mathsf q} \]
The sampling distribution of $\hat{\mathsf q}$ can be estimated using the formula
derived by  \cite{soper1913probable}; see the appendix for more details.

\subsection{Case Studies}
\label{sec:case_studies}
\begin{figure*}
\begin{center}
\includegraphics[trim={.6cm .4cm .6cm .45cm},clip,width=3.25in]{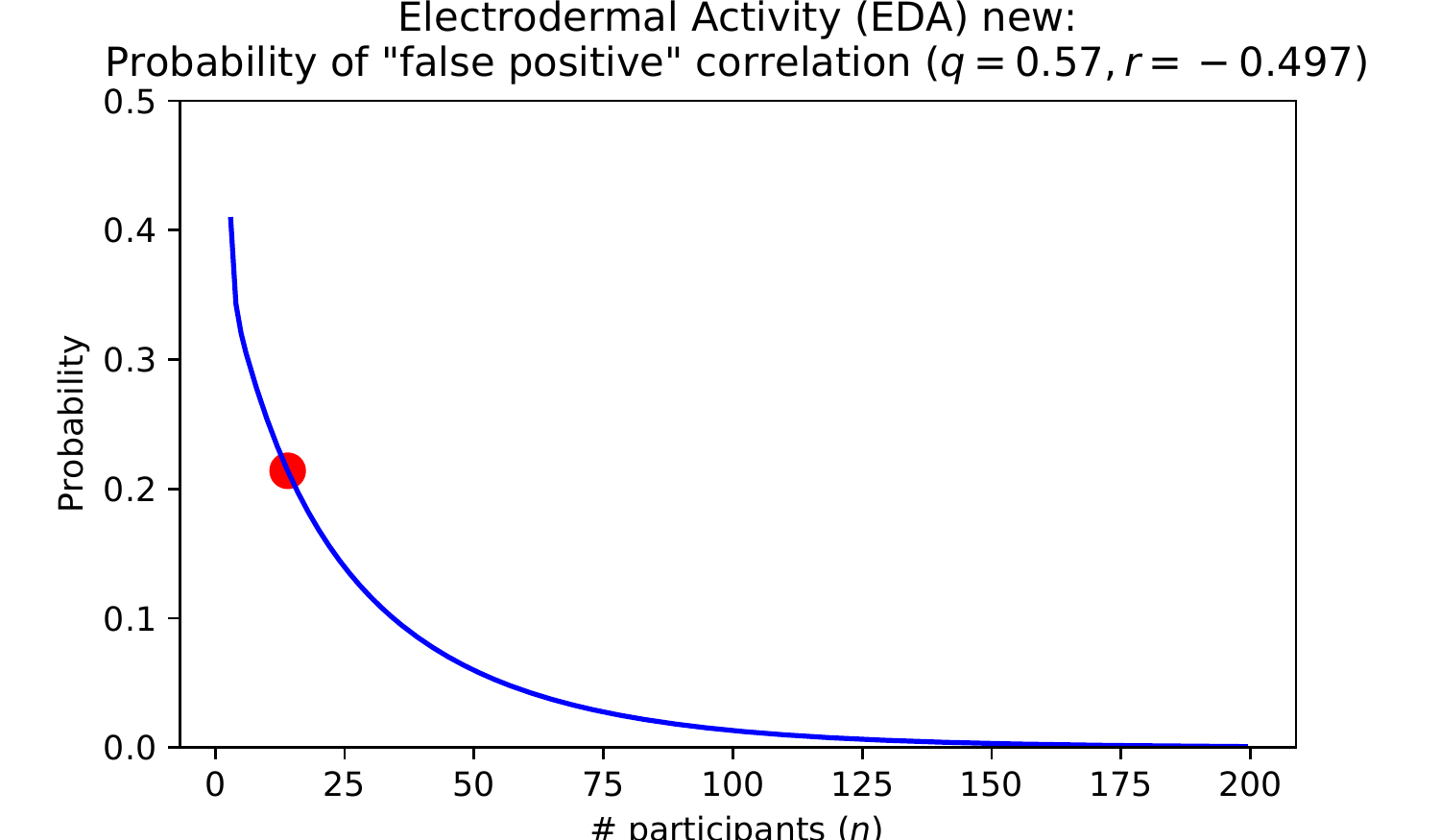}
\includegraphics[trim={.6cm .4cm .6cm .45cm},clip,width=3.25in]{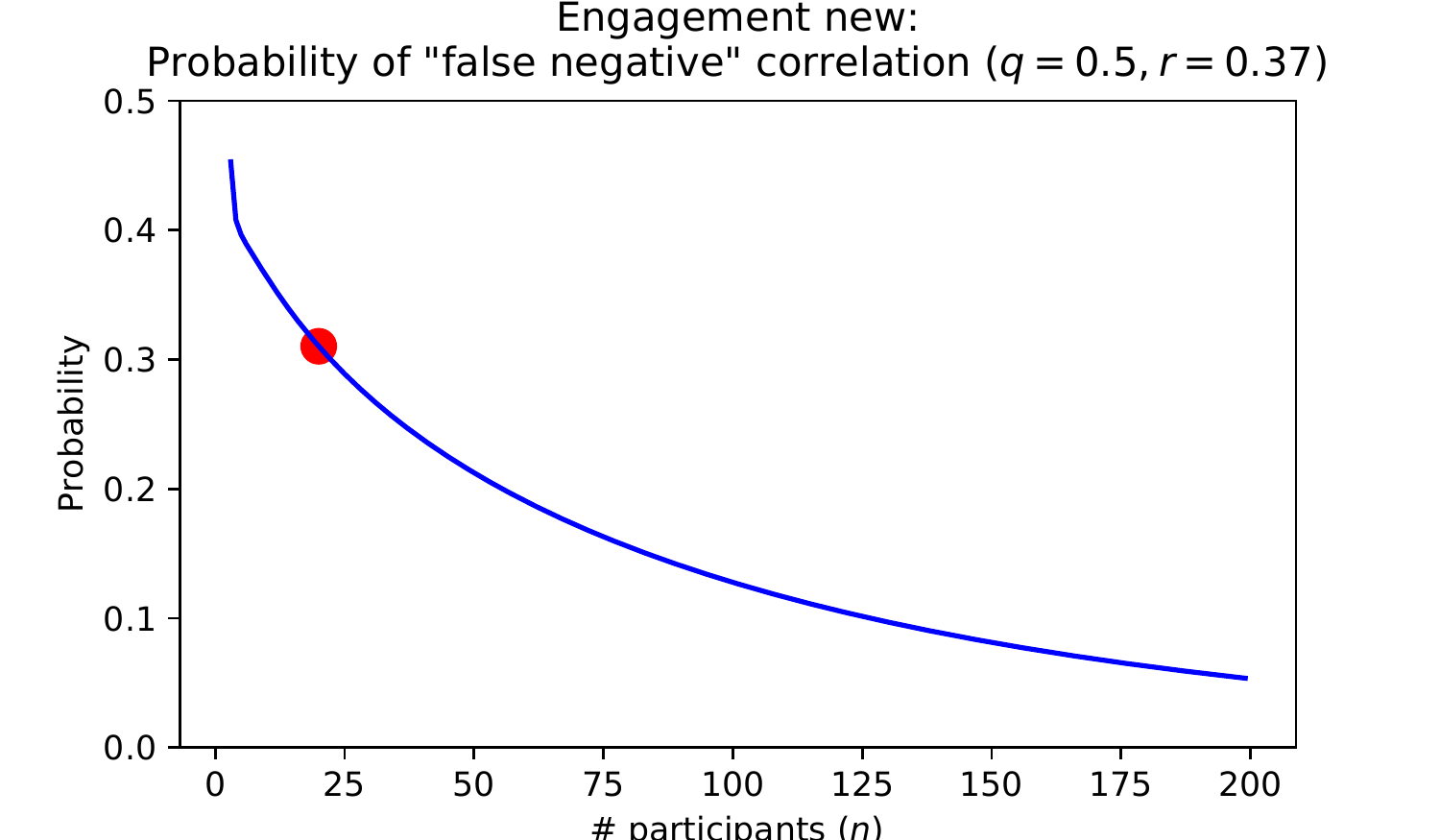}
\end{center}
\caption{Probability of false correlation, for  fixed $q$ and $r$, as a function of $n$.
The probability decreases as $n$ grows, but for small $n$ it can still be substantial. The red dots
indicate the $n$ from two recent behavioral studies
that used an automatic detector of affective state as a scientific instrument.  {\bf Left}:
Example inspired by a study on electrodermal activity  \cite{perugia2017electrodermal}. {\bf Right}: Example inspired by a study on student engagement  \cite{whitehill2014faces}. 
}
\label{fig:curves}
\end{figure*}

To put these theoretical results into perspective, we conducted
simulations based on two recent affective computing studies that
used automated detectors as scientific instruments. The first study ($n=14$),
by \cite{perugia2017electrodermal},
used an Empatica E4 wrist sensor to investigate how  electrodermal activity (EDA) ($U$)
is correlated with the subjects' emotions ($V$). The second study ($n=20$),
by  \cite{whitehill2014faces},
explored the relationship between student engagement ($U$), as measured by 
an engagement detector that analyzes static images of students' faces,
and test performance ($V$) in a cognitive skills training task.
In order to estimate the probability of a false correlation (in the sense
described above), we  need to know the accuracy of the automatic detector  -- i.e.,
the correlation $q$ between the automatic measurements and ground-truth of construct $U$ --
as well as the \emph{true} correlation $r$ between constructs $U$ and $V$.

{\bf Estimating $q$ and $r$}:
The value of $q$ can easily be estimated using cross-validation or other standard
procedures. For the first study (EDA), we use  the value $q=0.57$ 
reported in \cite{poh2010wearable} for cognitive tasks with a distill forearm sensor of EDA.
The value of $r$ is not knowable without access to ground-truth measurements; instead, we hypothesize
that the ground-truth correlation between $U$ and $V$ is exactly
what was estimated by the authors \cite{perugia2017electrodermal}
using the E4 sensor and emotion survey instruments: $r=-.497$.
For the second study (Engagement), we use the value $q=0.50$ reported in \cite{whitehill2014faces} that was obtained
using subject-independent cross-validation. For $r$, we use the correlation
obtained by the authors ($r=0.37$) when correlating test performance with
\emph{human}-labeled student engagement.

{\bf Results}: Plots of the probability of a false correlation (obtained 
from function $h$ derived above) as a function of the number of
participants $n$ are shown for each study (with their associated $q$ and $r$ values) 
in Figure  \ref{fig:curves}. The red dot in each graph
shows the actual number of participants from each experiment. Even for $n \geq 100$ subjects, the probability
is
non-trivial. For the values $n=14$ and $n=20$, these probabilities are substantial -- around $20\%$
for the EDA study and around $30\%$ for the Engagement study.
{\bf The possibility of a false correlation
is \emph{not} protected against by statistical significance testing} -- it is possible
for the estimated correlation between constructs $U$ and $V$ 
to be highly significant and yet have the wrong sign compared to the ground-truth correlation.
While this is almost always theoretically possible due to the inherent separation between
a construct and its measurement, the use in basic research of automatic detectors
that are trained to estimate another estimator can make this problem worse.

\section{Coverage of $\mathcal{T}^n$ when training a detector}
Given that \emph{which} vector $\bvr \in \mathcal{T}^n$ is obtained from 
an automatic detector $d$ can substantially impact the estimated correlation 
$\rho(\bvr, \bw)$ between constructs $U$ and $V$, it  could be useful to
\emph{average} the sample correlations  over \emph{many} vectors $\bvr^{(1)},\ldots,\bvr^{(m)}$ from
$\mathcal{T}^n$, i.e., to compute $\frac{1}{m} \sum_{i=1}^m \rho(\bvr^{(i)}, \bw)$.
This could help to reduce the variance of the estimator.
In this section we explore whether  it is feasible to generate many
different $\bvr^{(1)},\ldots,\bvr^{(m)}$, all with similar correlation $q$ with
ground-truth $\bu$, by training a set of automatic detectors $d^{(1)},\ldots,d^{(m)}$
using slightly different training configurations. In particular,
we varied: (1) the architecture (VGG-16 \cite{simonyan2014very} versus ResNet-50 \cite{he2016deep})
and (2) the random seed of weight initialization. Inspired by
recent work by Huang, et al. \cite{huang2017snapshot} on how 
an entire ensemble of detectors can be created during a  \emph{single} training run, we also 
varied (3) the number of training epochs, and saved snapshots of the trained detectors
at regular intervals.

During training, each detector's estimates $\bvr$ of the test labels evolves, and so
does the correlation between $\bvr$ and the ground-truth labels $\bu$.  However,
for the tasks we examined (described below),
the test correlations tend to stabilize over time, and  they converge to roughly
the same value even across different training runs and detection architectures.
Given a set of measurement vectors $\bvr^{(1)},\ldots,\bvr^{(m)}$ (produced by detectors $d^{(1)},\ldots,d^{(m)}$) whose accuracies (Pearson correlations with ground-truth)
are all approximately $q$, we can 
project them onto the $(n-3)$-sphere $\mathcal{T}^n$ of $0$-mean, unit-length vectors
whose correlation with $\bu$ is $q$. We can then visualize the 	``coverage'' of this
sphere by projecting it onto a 2-D plane, and also compare the coverage
to a random sample of $m$ elements of $\mathcal{T}^n$.

We explored the coverage of $\mathcal{T}^n$ for two automatic face analysis problems:
student engagement recognition  and
age estimation using the HBCU \cite{whitehill2014faces} (Engagement)
and GENKI \cite{GENKI} (Age) datasets, respectively; see Figure \ref{fig:datasets}
for labeled examples.
\begin{figure}
\begin{center}
\setlength{\tabcolsep}{2pt}
\begin{tabular}{cccccccc}
\multicolumn{8}{c}{Engagement (HBCU dataset \cite{whitehill2014faces})} \\
\includegraphics[width=.34in]{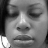} & 
\includegraphics[width=.34in]{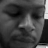} & 
\includegraphics[width=.34in]{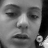} & 
\includegraphics[width=.34in]{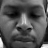} & 
\includegraphics[width=.34in]{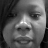} & 
\includegraphics[width=.34in]{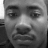} & 
\includegraphics[width=.34in]{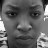} & 
\includegraphics[width=.34in]{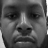} \\
1 & 1 & 2 & 2 & 3 & 3 & 4 & 4
\end{tabular}
\setlength{\tabcolsep}{2pt}
\begin{tabular}{cccccccc}
\multicolumn{8}{c}{Age (GENKI dataset \cite{GENKI})} \\
\includegraphics[width=.34in]{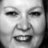} & 
\includegraphics[width=.34in]{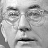} & 
\includegraphics[width=.34in]{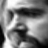} & 
\includegraphics[width=.34in]{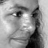} & 
\includegraphics[width=.34in]{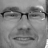} & 
\includegraphics[width=.34in]{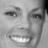} & 
\includegraphics[width=.34in]{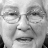} & 
\includegraphics[width=.34in]{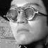} \\
55 & 60 & 30 & 38 & 31 & 30 & 70 & 30
\end{tabular}
\end{center}
\caption{Examples of face images used in the experiments in
Section \ref{sec:coverage_examples}. {\bf Top}: 
student engagement dataset \cite{whitehill2014faces}, in which engagement
is rated on a 1-4 scale.
{\bf Bottom}: images from the GENKI \cite{GENKI} dataset that were labeled with their perceived age (in years).}
\label{fig:datasets}
\end{figure}

\subsection{Detection architectures}
We examined two modern deep learning-based
visual recognition architectures -- VGG-16 \cite{simonyan2014very} and ResNet-50 \cite{he2016deep} -- to assess how much ``coverage'' of $\mathcal{T}^n$
each architecture can produce, as well as to compare
the variability of the resulting measurement vectors to each other.

\subsection{Training procedures}


{\bf Engagement detector}:
We performed $25$ training runs for each of the two network architectures (VGG-16, ResNet-50).
Training data consisted of $7629$ face images from 15 subjects of HBCU \cite{whitehill2014faces},
and testing data were $500$ images from the remaining 5 subjects.
(This corresponds to just one cross-validation fold from the original study \cite{whitehill2014faces}.)
Optimization was performed using SGD 
for $10000$ iterations, and the network weights were saved
every $1000$ iterations. In total, this produced $250$ detectors. 
The average correlation (over all $250$ detectors)
between the detectors'  test outputs $\bvr$ and ground-truth $\bu$ was $0.61$ (s.d.~$0.081$) for VGG-16 and $0.64$ (s.d.~$0.009$) for ResNet-50. 
Inspired by \cite{huang2017snapshot}, we also tried both cosine and
triangular \cite{smith2017cyclical} learning rates. However, in pilot testing we found that these
delivered worse accuracy than exponential learning rate decay and 
we abandoned the approach.

{\bf Age detector}:
We performed $25$ training runs for VGG-16 and ResNet-50 using SGD for $10000$ with snapshots every $1000$
iterations, as for engagement recognition. This produced $250$ detectors.
Training data consisted of $31040$ face images of the GENKI dataset \cite{GENKI},
and testing data consisted of $500$ face images.  The average
correlation of the automatic measurements with ground-truth on the test set was 
$0.595$ (s.d.~$0.036$) for VGG-16 and $0.60$ (s.d.~$0.014$) for ResNet-50.

\subsection{Visualizing elements of the $(n-3)$-sphere $\mathcal{T}^n$}
\label{sec:viz_disc}
Given a set of $m$ trained detectors, we can sample vectors of age/engagement estimates
$\bv^{(1)},\ldots,\bv^{(m)} \in \mathbb{R}^n$ whose correlation with
ground-truth $\bu$ is approximately $q$. Then we can visualize how these
vectors ``cover'' the $(n-3)$-sphere $\mathcal{T}^n$ using the following
procedure:
\begin{enumerate}
\item Normalize $\bu$, as well as each $\bv^{(j)}$, to have $0$-mean and unit-length.
\item Compute an orthonormal basis $\mathbf{B}$ (e.g., using a QR decomposition) so that (a) the first component of $\mathbf{B} \mathbf{x}$ is $0$
for every $\mathbf{x} \in \mathcal{H}$, and (b)
$\mathbf{B} \bu = (0,1,0,0,\ldots,0)$ (see Equation \ref{eqn:transformed_u}).
\item Project each $\bv^{(j)}$ onto the new basis $\mathbf{B}$.
By construction, the first component of each projection will be $0$ and the
second component will be $q=\rho(\bv^{(j)}, \bu)$.
\item Define each $\mathbf{x}^{(j)}$ to be the last $n-2$ components
of vector $\mathbf{B} \bv^{(j)}$.
\item Project the $\{ \mathbf{x}^{(j)} \}$ onto the two principal axes
obtained from principal component analysis (PCA).
\end{enumerate}
Since the 2-D projection of a 0-centered sphere onto any orthonormal projection  is a disc, 
the output of the procedure above is a set of points that lie on a disc of radius
$\sqrt{1-q^2}$.

\subsection{Generating random vectors on $\mathcal{T}^n$}
\label{sec:sampling}
In order to assess how evenly  the sphere $\mathcal{T}^n$ is ``covered''
by the vectors obtained from the automatic detectors, we can  generate
random vectors of $\mathcal{T}^n$ and likewise project them onto a 2-D disc.
We generate each such vector as follows:
\begin{enumerate}
\item Sample each $\mathbf{z}_i$ ($i=1,\ldots,n-2$) from a  standard normal distribution.
\item For $i=1,\ldots,n-2$, set $\mathbf{\hat u}_i = \frac{\sqrt{1-q^2} \times {\mathbf z}_i}{\sqrt{ \sum_{i'=1}^{n-2} {\mathbf z}_{i'}^2 }}$.
\end{enumerate}
We then project the vectors in the set  $\{ \mathbf{\hat u}^{(j)} \}$ onto the two principal axes
obtained from PCA. To enable a fair comparison between the variances of the randomly
generated elements of $\mathcal{T}^n$ and those obtained from the trained detectors,
we run PCA \emph{separately} for each set.

\subsection{Results}
\label{sec:coverage_examples}

\begin{figure*}
\begin{center}
\includegraphics[trim={.6cm .6cm .6cm .6cm},clip,width=3.25in]{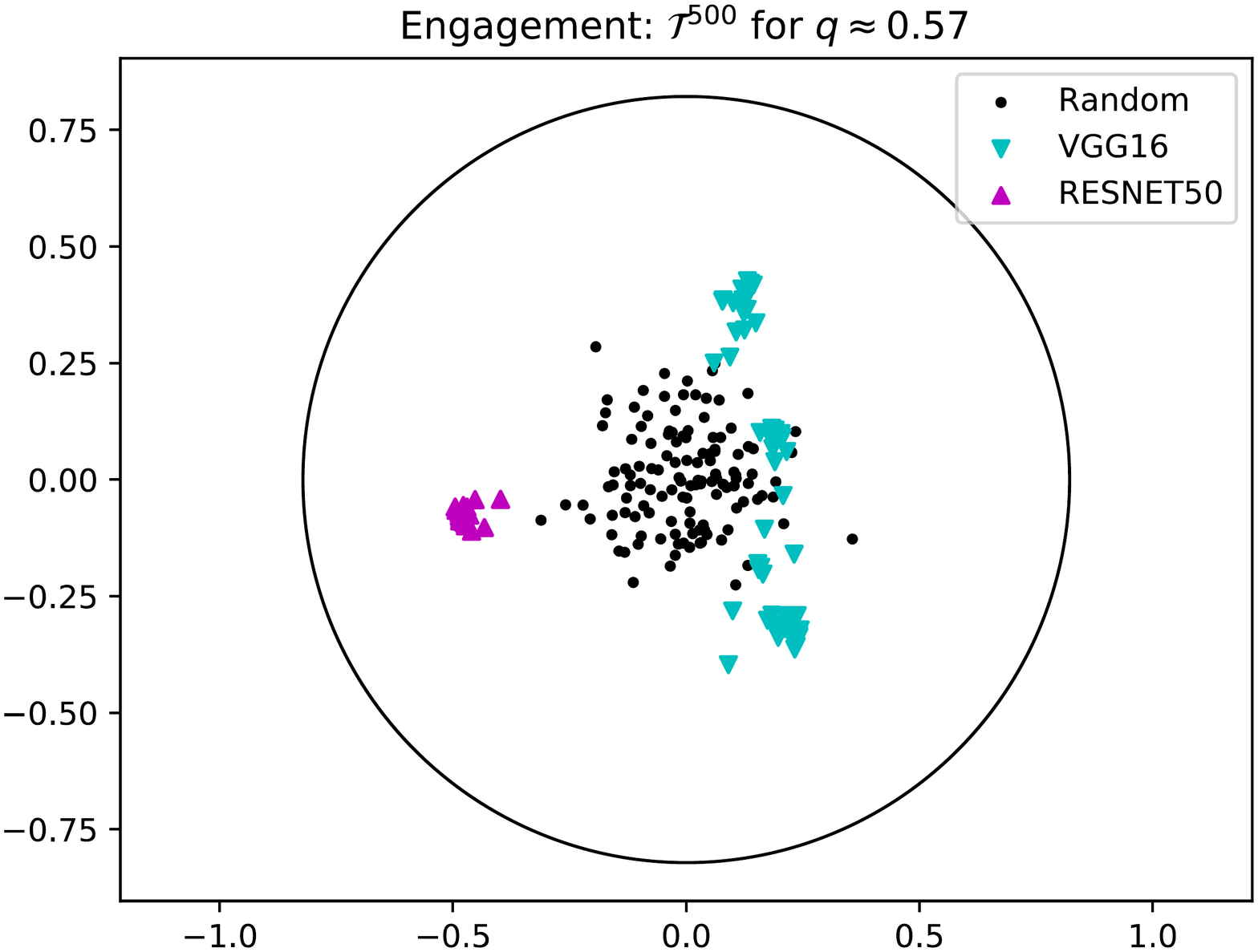}
\includegraphics[trim={.6cm .6cm .6cm .6cm},clip,width=3.25in]{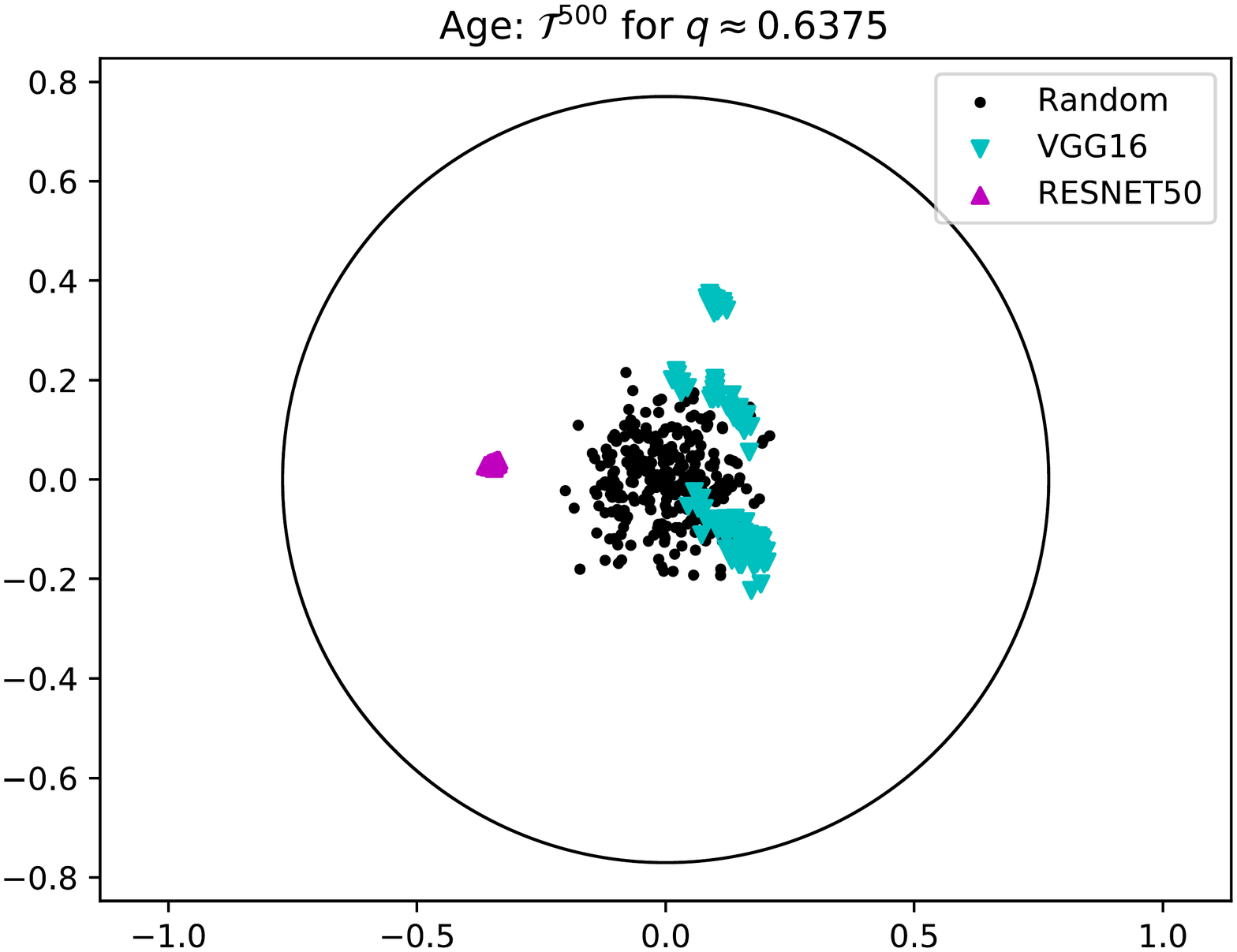}
\end{center}
\caption{Coverage of the $\mathcal{T}^n$ sphere (projected onto 
2 dimensions using PCA) from different neural networks trained to predict student engagement
({\bf left}) and age ({\bf right}) from face images.
We used either VGG-16 or ResNet-50 neural networks.
For comparison,
we also sampled random vectors using the procedure from Section \ref{sec:sampling}.
}
\label{fig:coverage}
\end{figure*}

We projected all vectors $\{ \bv^{(j)} \}$ whose correlation with $\bu$ was between
$0.575$ and $0.625$ for Engagement recognition and $0.625$ and $0.65$ and Age estimation; this amounted to $50\%$ of the engagement detectors and $45\%$ of the age detectors.
The projections  are shown for each task
in Figure \ref{fig:coverage}.
First, we observe that there is some ``spread'' -- the measurement vectors occupy different clusters
on the sphere. This indicates that the same training data can still yield automatic measurements $\bv$
on testing data whose correlations with each other is far less than 1. In fact, for engagement
recognition, the minimum
correlation, over all pairs $(\bv^{(j)}, \bv^{(j')})$, was $0.64$.

For engagement recognition, the VGG-16 based measurements and the ResNet-50 based  measurements
each resided within their own clusters on the sphere, and these clusters did not overlap. This suggests
that, even though both architectures yielded similar overall accuracies,  they
are making different kinds of estimation errors on the test set. Interestingly for both age estimation and engagement recognition we can see that VGG-16 has a bigger ``spread'' compared to ResNet-50. We speculate this might be due to VGG-16 ($138$ million) having significantly more  parameters compared to ResNet-50 ($25$ million), thus enabling it to learn more varied features.



Finally, a comparison of the variance between the automatic measurements $\bv^{(1)},\ldots,\bv^{(m)}$
and random samples from $\mathcal{T}^n$ indicates that varying the training configuration (architecture,
hyperparameters) provides only limited ability to cover the sphere: the variance in the vectors,
as quantified as the sum of the trace
of their covariance matrix, was statistically significantly
less compared to randomly sampled points on $\mathcal{T}^n$ ($p<0.01$, 1-tailed, Monte Carlo simulation).



\section{Conclusions}
\label{sec:conclusions}
Advances in machine perception present a powerful opportunity
to create new  scientific instruments that can benefit  basic research in sociobehavioral
sciences.  However, since detectors are often
trained to estimate existing measures, which are already only an estimate of underlying constructs,
then these instruments are essentially one step further removed from ground-truth. For this reason,
it  is important to interpret results obtained with them with care.

In this paper, we investigated
how measurements of construct $U$ obtained with an automatic detector
can impact the estimated correlation between $U$ and another construct $V$. We showed that:
(1) The set of $0$-mean unit-length $n$-vectors with a fixed Pearson product-moment correlation  $q$ to
vector $\bu$ is a $(n-3)$-sphere $\mathcal{T}^n$ embedded in $\mathbb{R}^n$.
(2) If the correlation between automatic measurements $\bvr$ and the ground-truth measurements is $q$;
if the true correlation between $U$ and $V$ is $r$; and if $\bvr$ is sampled uniformly from $\mathcal{T}^n$;
then the expected sample correlation obtained with the automatic detector is $qr$.
(3) The probability of a ``false correlation'', i.e., a sample correlation
between constructs $U$ and $V$ whose sign differs from the true correlation, is monotonically decreasing
in $n$ (number of participants) and also monotonically decreasing in $q$ (accuracy of the detector). 
These probabilities can be non-trivial for small values of $n$ that are nonetheless sometimes found in contemporary research
using automatic facial expression and affect detectors. Moreover, the danger of a false correlation
is not eliminated through statistical significance testing.
(4)  We explored empirically how efficiently multiple
neural network-based detectors of age and student engagement, when trained using different architectures
and hyperparameters but the same training data, can ``cover'' the sphere $\mathcal{T}^n$.

In practice, our results suggest that, particularly when the number of participants is small and/or
the accuracy of the detector is modest, it is important to consider the possibility of a false correlation,
or at least a skewed correlation (by factor $q$), when drawing scientific conclusions.

{\bf Limitation and future work}: In our study we assumed that
$\bvr$ is a random sample from the uniform distribution over
$\mathcal{T}^n$ -- this expresses the idea that \emph{a priori} we may have no idea which \emph{particular}
element of $\mathcal{T}^n$ detector $d$ will return. In reality, however, detectors have biases -- e.g.,
due to head pose, lighting conditions, training set composition, etc. -- and these can affect
which element of $\mathcal{T}^n$ is obtained.

{\bf Acknowledgements}:
This research was supported by a Cyberlearning grant from the National Science Foundation (grant no.~\#1822768).


\section{Appendix}
\subsection{Proof of Proposition 2}
We prove the proposition for the case that $r <0$; the case for $r > 0$ is similar.

From Section III, we have that
\[ \rho(\bvr,\bw) = qr + \mathsf{\hat u}_3 \sqrt{1-r^2} \]
%
Since each $\hat{\mathsf u}_i$ ($i=3,4,\ldots,n$) is a coordinate on an $(n-3)$-sphere,
it can be re-parameterized  \cite{muller1959note} by sampling $n-2$ standard normal random variables and normalizing, i.e.:
\[
\hat{\mathsf u}_i = \frac{\sqrt{1-q^2} \times {\mathsf z}_i}{\sqrt{ \sum_{j=3}^n {\mathsf z}_j^2 }}
\]
where each ${\mathsf z}_i \sim \mathbb{N}(0,1)$. 
A false positive correlation thus occurs  when 
$\hat{u}_3$ is at least $c = |qr|/\sqrt{1-r^2}$ more than its expected value $qr$:
\[
\textrm{Pr}[\hat{\mathsf u}_3 \geq c] = \textrm{Pr}\left[\frac{\sqrt{1-q^2} \times {\mathsf z}_3}{\sqrt{ \sum_{j=3}^n {\mathsf z}_j^2 }} \geq c\right]
\]
Due to the inequality, we must handle the cases that ${\mathsf z}_3 \geq 0$ and ${\mathsf z}_3 < 0$ separately.
Note that the latter case contributes 0 probability since $c\geq 0$ and $q>0$.
Also, since ${\mathsf z}_3$ is a standard normal random variable, $\textrm{Pr}[{\mathsf z}_3\geq 0]=0.5$.
\begin{eqnarray*}
\lefteqn{\textrm{Pr}[\hat{u}_3 \geq c]} && \\
 &=& \textrm{Pr}\left[\frac{\sqrt{1-q^2} \times {\mathsf z}_3}{\sqrt{ \sum_{j=3}^n {\mathsf z}_j^2 }} \geq c\ \Big|\ {\mathsf z}_3 \geq 0\right] \textrm{Pr}[{\mathsf z}_3\geq 0] + \\
 && \textrm{Pr}\left[\frac{\sqrt{1-q^2} \times {\mathsf z}_3}{\sqrt{ \sum_{j=3}^n {\mathsf z}_j^2 }} \geq c\ \Big|\ {\mathsf z}_3 < 0\right] \textrm{Pr}[{\mathsf z}_3< 0] + \\
 &=& \frac{1}{2} \textrm{Pr}\left[\frac{\sqrt{1-q^2} \times {\mathsf z}_3}{\sqrt{ \sum_{j=3}^n {\mathsf z}_j^2 }} \geq c\ \Big|\ {\mathsf z}_3 \geq 0\right] + 0 \\
 &=& \frac{1}{2} \textrm{Pr}\left[(1-q^2){\mathsf z}_3^2 \geq c^2 \sum_{j=3}^n {\mathsf z}_j^2\right] \\
 &=& \frac{1}{2} \textrm{Pr}\left[(1-q^2-c^2){\mathsf z}_3^2 \geq c^2 \sum_{j=4}^n {\mathsf z}_j^2\right] \\
 &=& \frac{1}{2} \textrm{Pr}\left[{\mathsf z}_3^2 \geq \frac{c^2}{(1-q^2-c^2)} \sum_{j=4}^n {\mathsf z}_j^2\right]
\end{eqnarray*}
For $n>3$,
each side of the inequality is a sum of squared normally distributed random variables, i.e., a $\chi^2$-random
variable (though with different degrees of freedom).  We can thus rewrite this probability as
\begin{eqnarray*}
\textrm{Pr}[\hat{u}_3 \geq c] &=&\frac{1}{2} \textrm{Pr}\left[\chi^2_1 \geq \left(\frac{c^2}{1 - q^2 - c^2}\right) \chi^2_{(n-3)} \right] \\
                            &=& \frac{1}{2} \int_0^\infty f_1(t) F_{n-3}\left(\frac{1 - q^2 - c^2}{c^2} t\right) dt \\
                            &\doteq& h(n,q,r)
\end{eqnarray*}
where $\chi^2_1$ and $\chi^2_{(n-3)}$ are $\chi^2$ random variables with $1$ and $(n-3)$ degrees
of freedom, respectively.
The probability is equivalent to the integral because, for any value $t$ of
the $\chi_1^2$ variable, we require that the $\chi_{n-3}^2$ variable be less than
$t$ (after applying a scaling factor).  To our knowledge, there is no closed formula for this integral,
but we can compute it numerically.  For $n=3$, we have 
\begin{eqnarray*}
\textrm{Pr}[\hat{u}_3 \geq c] &=& \frac{1}{2} \textrm{Pr}\left[(1-q^2 - c^2) {\mathsf z}_3^2 \geq 0\right] \\
                           &=& \frac{1}{2} \textrm{Pr}[c^2\leq 1-q^2]
\end{eqnarray*}
since a $\chi^2$-random variable is  non-negative, and
where the probability of $c^2\leq 1-q^2$ is $1$ if the inequality is true and $0$ otherwise.


\subsection{Proof of Proposition 3}
For convenience, define $\alpha = \frac{1 - q^2 - c^2}{c^2}$.
\begin{eqnarray*}
\lefteqn{h(n+1,q,r) - h(n,q,r)} && \\
&=&
\int_0^\infty \left[f_1(t) F_{(n+1)-3}\left(\alpha t\right) -
f_1(t) F_{n-3}\left(\alpha t\right)\right] dt \\
&=& 
\int_0^\infty f_1(t) \left[F_{n-2}\left(\alpha t\right) -
                           F_{n-3}\left(\alpha t\right)
                           \right] dt \label{eqn:monotone}
\end{eqnarray*}
Ghosh \cite{ghosh1973some} proved that, for any fixed $t>0$, $\textrm{Pr}[\chi_k^2 > t]$ is monotonically increasing in
the degrees of freedom $k$; hence, $F_k(t)$ is monotonically decreasing in $k$. 
Therefore,
$F_{n-2}\left(\alpha t\right) - F_{n-3}\left(\alpha t\right) < 0$
for all $t$.  Since $f_k$ is a non-negative
function for all $k$, then the integral in Equation \ref{eqn:monotone} must be negative; hence,
$h$ is monotonically decreasing in $n$ for every $c>0$ and $q\in(0,1]$.

\subsection{Proof of Proposition 4}

First, we show that $\alpha$ is monotonically decreasing in $q^2$:
\begin{eqnarray*}
\alpha(q) &=& \frac{1 - q^2 - c^2}{c^2} \\
       &=& \frac{1 - q^2 - q^2 r^2/(1-r^2)}{q^2 r^2/(1-r^2)} \\
       &=& \frac{(1-r^2)(1-q^2) - q^2 r^2}{q^2 r^2} \\
       &=& \frac{1 - r^2 - q^2}{q^2 r^2} \\
       &=& \frac{1 - r^2}{q^2 r^2} - \frac{1}{r^2}
\end{eqnarray*}
The first term is monotonically decreasing in $q^2$, and the second term is constant in $q^2$.

Next, let $\epsilon$ be a positive real number such that $q+\epsilon \leq 1$:
\begin{eqnarray*}
\lefteqn{h(n,q+\epsilon,r) - h(n,q,r)} && \\
&=&
\int_0^\infty f_1(t) F_{n-3}\left(\alpha(q+\epsilon) t\right) dt - \\
&& \int_0^\infty f_1(t) F_{n-3}\left(\alpha(q) t\right) dt \\
&=& 
\int_0^\infty f_1(t) \left[F_{n-3}\left(\alpha(q+\epsilon) t\right) -
                           F_{n-3}\left(\alpha(q) t\right)
                           \right] dt
\end{eqnarray*}
Since $F_{n-3}$ is monotonically \emph{increasing}, then the expression in brackets is negative.
Since $f_1$ is non-negative, then the entire integral must be less than 0.

\section{Sampling distribution $\textrm{Pr}(\hat{\mathsf q}\ |\ q, n)$}
The sampling distribution can  be computed exactly \cite{fisher1915frequency}, but this is computationally feasible only for small $n$. Hence, we use the approximation from Soper \cite{soper1913probable}: Let $q$ denote the population Pearson correlation coefficient, and let $\hat{\mathsf q}$ denote the sample correlation from $n$ data. Then

\begin{eqnarray*}
\textrm{Pr}(\hat{\mathsf q}\ |\ q,n) &\propto&
(1-\hat{\mathsf q})^{m_1}  (1+\hat{\mathsf q})^{m_2}\\
m_1&=& \frac{1}{2} (\lambda-1)(1-\mu_q)-1\\
m_2&=& \frac{1}{2} (\lambda-1)(1+\mu_q)-1\\
\lambda&=& (1-\mu_q^2)/\sigma_q^2\\
\sigma_q&=& \frac{(1-q^2)}{\sqrt{n}}  \left(1+\frac{(1+5.5q^2)}{2n}\right)\\
\mu_q&=& \sqrt{q^2 - \frac{c}{n}-\frac{c(1+5q^2)}{2n^2}}\\
c&=& q^2 (1-q^2)
\end{eqnarray*}

\bibliography{paper}
\bibliographystyle{icml2019}

\end{document}